\documentclass{article}

\usepackage{iclr2027_conference,times}

\makeatletter
\def\@iclrnotice{Preprint}
\makeatother

\usepackage[T1]{fontenc}
\usepackage{amsmath,amssymb,amsfonts,amsthm}
\usepackage{algorithm}
\usepackage{algorithmic}
\usepackage{booktabs}
\usepackage[table]{xcolor}
\usepackage{graphicx}
\usepackage{float}
\usepackage{microtype}
\usepackage{url}
\usepackage[hidelinks]{hyperref}
\usepackage{tabularx}
\usepackage{enumitem}

\newtheorem{proposition}{Proposition}

\newtheorem{lemma}[proposition]{Lemma}

\newtheorem{theorem}[proposition]{Theorem}

\DeclareMathOperator{\clip}{clip}
\DeclareMathOperator{\polar}{polar}
\DeclareMathOperator{\NS}{Newton-Schulz}
\newcommand{\fnorm}[1]{\left\lVert #1\right\rVert_{\mathrm F}}
\newcommand{\inner}[2]{\left\langle #1,#2\right\rangle_{\mathrm F}}

\title{Inner Momentum for Differentially Private Muon}

\author{
\begin{tabular}{cccc}
\textbf{Bishnu Bhusal}$^{1,2}$ &
\textbf{Minh Vu}$^{1}$ &
\textbf{Ben Southworth}$^{1}$ &
\textbf{Geigh Zollicoffer}$^{1,3}$
\end{tabular}
\\[3pt]
\begin{tabular}{cc}
\textbf{Rohit Chadha}$^{2}$ \qquad &
\textbf{Manish Bhattarai}$^{1}$
\end{tabular}
\\[5pt]
{\footnotesize
$^{1}$Los Alamos National Laboratory \qquad
$^{2}$University of Missouri \qquad
$^{3}$Georgia Institute of Technology
}
}

\newcommand{\rc}[1]{{\color{blue}#1}}

\iclrfinalcopy

\begin{document}

\maketitle

\begin{abstract}

Differentially private training clips each per-example gradient before adding noise. This clipping is radial for each example, yet unequal clipping factors can distort the relative singular-vector geometry of their average. Muon is particularly exposed to this effect, since its update is an approximate polar factor $UV^\top$ that depends only on the singular vectors that clipping can shift. To curb this degradation, we propose averaging each sampled example's Muon gradient over the current model and a short history of recent models before clipping. The clipped batch matrix then separates into a common rescaling and a covariance residual $R$ between sampled gradients and clipping values, with $\|R\|_F \le \sigma_\lambda \sigma_G$, bounding the clipping-induced distortion directly. We further show that a finite Newton–Schulz iteration preserves the polar factor of its input under these spectral conditions, confirming that our correction survives orthogonalization. In private GPT-2 fine-tuning on E2E and DART at $\epsilon \in \{1, 2, 4, 8\}$, DP-Muon-IM improves BLEU and ROUGE-L over DP-Muon in every seed-matched comparison, and non-private diagnostics show 2--4\% lower pre-noise polar error.

\end{abstract}

\section{Introduction}\label{sec:intro}

Training a neural network can reveal information about individual training records. Differential privacy (DP) bounds this leakage by ensuring that a model's output distribution changes only slightly depending on whether any single record was used in training~\citep{dwork2014algorithmic}. A standard approach in private deep learning is DP stochastic gradient descent (DPSGD). It clips each per-example gradient before averaging, which limits the influence of one record, and adds Gaussian noise calibrated to that limit~\citep{abadi2016deep}. Both operations can affect utility. Noise perturbs the averaged gradient, while clipping can change its direction by rescaling per-example gradients unequally~\citep{chen2020geometric}.

This work studies how clipping affects Muon~\citep{jordan2024muon}, an optimizer that approximately orthogonalizes matrix-valued momentum using a small number of Newton--Schulz iterations. Muon has recently been used alongside AdamW~\citep{AdamW} to train modern large language models, including DeepSeek-V4~\citep{deepseek2026v4} and Kimi K2 via MuonClip~\citep{kimi2025k2}, with Muon-based training outperforming AdamW baselines on small transformers and Moonlight~\citep{liu2025muon}. Motivated by these gains, we ask whether Muon's advantages persist under the per-example gradient clipping and Gaussian noise required for differentially private training. In fact, per-example clipping affects Muon by changing the relative weights of the gradients in the batch average. As shown in Fig.~\ref{fig:intuition}, when clipping factors differ across examples, the clipped average can have different singular-vector bases, and hence a different polar factor~\citep{chen2020geometric,xiao2023theory}. DP-Muon~\citep{kim2026dpmuon} adds Gaussian noise to this clipped average and then applies momentum and Newton–Schulz. Because the Newton--Schulz iteration preserves the polar factor of its input in exact arithmetic, it cannot correct clipping-induced changes to that polar factor.

\begin{figure}
    \centering
    \includegraphics[width=0.95\linewidth]{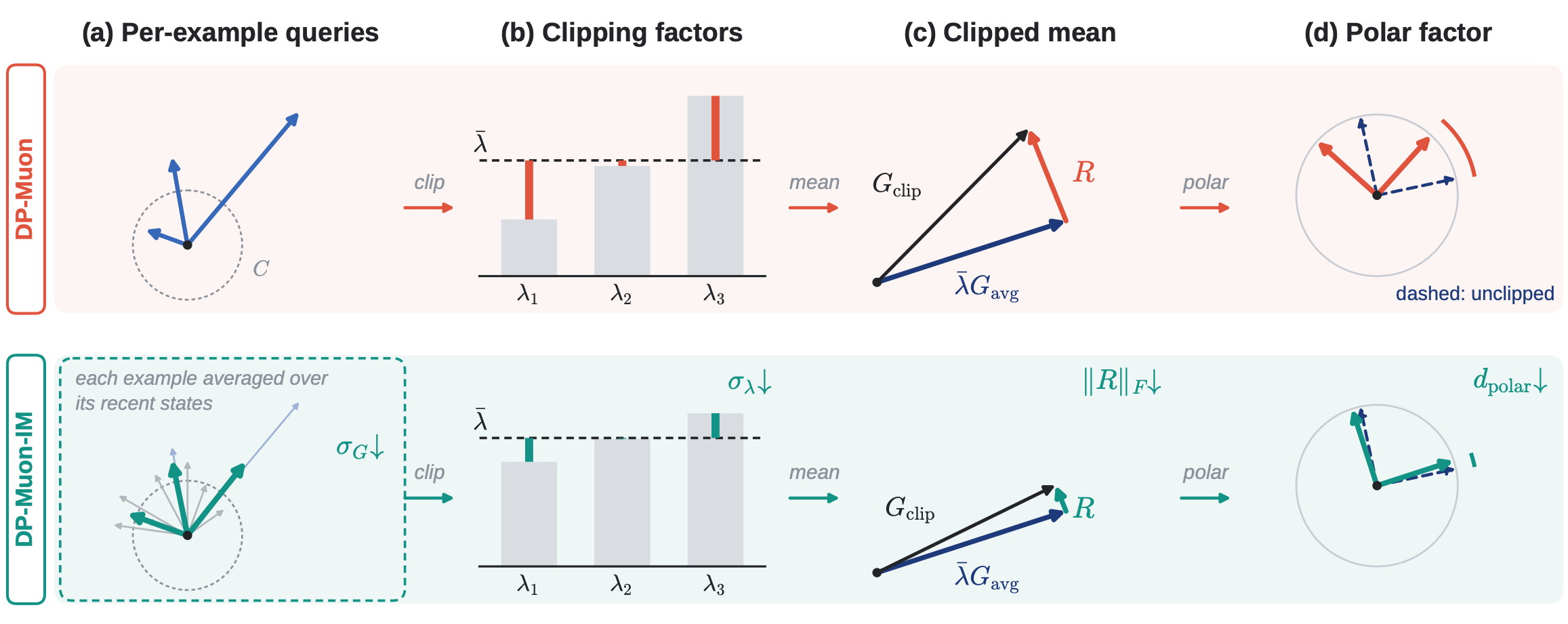}
    \caption{\textbf{Intuition for DP-Muon with inner momentum (DP-Muon-IM).} \textbf{(a)} DP-Muon-IM averages each example's Muon gradients at the current and recent private model states before flat clipping. \textbf{(b-c)} Unequal clipping factors introduce a residual $R$ relative to a common rescaling of the batch mean. IM averages per-example gradients before this reweighting. \textbf{(d)} The residual can change the polar factor. We measure residual size and polar error separately; the utility experiments evaluate the complete method.}
    \label{fig:intuition}
    \vspace{-0.3cm}
\end{figure}

Instead, we propose to smooth each sampled example’s Muon gradient across the current and recent private model states before clipping. Under the stated variance and temporal-covariance assumptions, averaging can tighten a bound on the expected clipping-residual norm. We measure the realized residual and polar error separately. The resulting method, DP-Muon-IM, adapts the inner-momentum approach of \citet{xiao2023theory} to realize that intuition. We further formalize this idea by decomposing the clipped average into a common-rescaling term, which preserves the desired polar factor, and a residual governed by variation in the gradients and clipping factors. Privacy accounting remains unchanged: each example contributes a single query clipped to the same radius as in the baseline Gaussian mechanism, with the same sampling procedure and release schedule. In GPT-2 fine-tuning on E2E and DART, DP-Muon-IM improves BLEU and ROUGE-L over DP-Muon.
 
Our contributions are as follows.

 \vspace{-0.2cm}
\begin{itemize}[leftmargin=6pt]
    \item Section~\ref{sec:clipping} explains how clipping can alter the matrix geometry relevant to Muon by decomposing the clipped average into a common rescaling and a residual $R$, with $\fnorm{R}\le\sigma_\lambda\sigma_G$. We also show that the stated finite Newton--Schulz iteration preserves its immediate input's polar factor in exact arithmetic, so it cannot remove polar error present at that stage.

   \item Section~\ref{sec:method} introduces DP-Muon-IM, which averages each sampled example's gradients at current and recent private model states before clipping. The sampling rule, clipping radius, noise calibration, release schedule, and baseline per-run privacy accounting remain unchanged.

    \item Section~\ref{sec:experiments} presents private GPT-2 experiments showing that DP-Muon-IM improves both BLEU and ROUGE-L scores on E2E and DART compared with DP-Muon. We report E2E and DART results at target $\epsilon\in\{1,2,4,8\}$, and a single-seed history-length ablation quantifies the tradeoff between generation quality, training time, and memory.
\end{itemize}

\textbf{Organization.}
Section~\ref{sec:setting} introduces notation and background. Section~\ref{sec:clipping} analyzes clipping-induced distortion, and Section~\ref{sec:method} describes DP-Muon-IM. Section~\ref{sec:experiments} reports the experiments. Sections~\ref{sec:related-work}--\ref{sec:conclusion} cover related work, discussion, and conclusions.

\vspace{-0.2cm}
\section{Preliminaries}
\label{sec:setting}


\vspace{-0.2cm}
\subsection{Problem setup and notation}
\label{sec:problem-setup}
\vspace{-0.2cm}

Let $\mathcal D=\{x_i\}_{i=1}^n$ be the private training set, let $\theta$ denote the model parameters, and let $\ell(\theta;x_i)$ be the loss on example $x_i$. Training seeks to minimize the empirical risk
\begin{equation}
F(\theta)=\frac{1}{n}\sum_{i=1}^n\ell(\theta;x_i).
\label{eq:empirical-risk}
\end{equation}
At training step $t$, the optimizer samples a minibatch $S_t$ (also called a \emph{lot} in DP training), and we write $m=|S_t|$ for its realized size. The per-example gradient is
\begin{equation}
G_t^{(i)}=\nabla_\theta\ell(\theta_t;x_i).
\label{eq:complete-gradient}
\end{equation}
Throughout we treat a gradient as a vector over all trainable parameters. However, algorithmically Muon operates on matrix representations of specific subsets of parameters. Thus let $\mathcal W$ be the set of hidden-layer matrices operated on by Muon, with each matrix corresponding to a specific subset of parameter indices. For any $W\in\mathcal W$, the corresponding block of $G_t^{(i)}$ is implicitly converted to its matrix form,
\begin{equation}
G_{t,W}^{(i)}=\nabla_W\ell(\theta_t;x_i) \in\mathbb{R}^{n^{W}_{\mathrm{out}}\times n^{W}_{\mathrm{in}}}.
\label{eq:sample-gradient}
\end{equation}
Muon acts separately on each $W\in\mathcal W$. Non-Muon parameters are omitted from the blockwise notation. If present, their gradients remain in $G_t^{(i)}$, receive the same clipping-and-noise mechanism, and are updated by the auxiliary optimizer. The norm $\lVert G_t^{(i)}\rVert_2$ treats all parameter blocks as one vector.

\vspace{-0.2cm}
\subsection{Differential privacy and private optimization}
\vspace{-0.2cm}

Two datasets are neighbors if they differ by the addition or
removal of one record. A randomized algorithm $\mathcal M$
is $(\epsilon,\delta)$-differentially private~\citep{dwork2014algorithmic}
if, for every neighboring pair $\mathcal D,\mathcal D'$ and
every measurable output set $\mathcal E$,
\begin{equation}
\Pr[\mathcal M(\mathcal D)\in\mathcal E]
\le e^\epsilon
\Pr[\mathcal M(\mathcal D')\in\mathcal E]+\delta,
\label{eq:dp-definition}
\end{equation}
where $\epsilon$ is the privacy budget and $\delta$ allows a small additive relaxation; smaller values give stronger privacy.

To achieve this in optimization, given a clipping radius $C>0$, flat clipping bounds each example's contribution by rescaling its complete gradient:
\begin{equation}
\clip_C(G_t^{(i)})=\lambda_t^{(i)}G_t^{(i)},
\qquad
\lambda_t^{(i)}=\min\left\{1,\frac{C}{\lVert G_t^{(i)}\rVert_2}\right\},
\label{eq:clip}
\end{equation}
with $\lambda_t^{(i)}=1$ when $G_t^{(i)}=0$. For matrix-level equations, $\clip_C(G_{t,W}^{(i)})$ denotes the $W$ block of $\clip_C(G_t^{(i)})$, with the clipping factor still determined by the full gradient norm. Letting $B$ be the (expected) lot size, the private release is
\begin{equation}
\widetilde G_t
=\frac{1}{B}\left(\sum_{i\in S_t}\clip_C(G_t^{(i)})
+\sigma_{\mathrm{DP}}C Z_t\right),
\qquad Z_t\sim\mathcal N(0,I),
\label{eq:private-release}
\end{equation}
with $\widetilde G_{t,W}$ denoting the matrix block of $\widetilde G_t$ for $W\in\mathcal W$. Clipping controls sensitivity, Gaussian noise provides privacy, and all subsequent optimizer operations are post-processing.

\vspace{-0.2cm}
\subsection{Muon and DP-Muon}
\vspace{-0.2cm}

Muon updates each hidden-layer matrix separately~\citep{jordan2024muon}. Fix $W\in\mathcal W$, and let $\widetilde G_{t,W}$ be its block of the private release in Eq.~\eqref{eq:private-release}. DP-Muon \citep{kim2026dpmuon} applies
\begin{equation}
\begin{aligned}
M_{t,W}&=\mu M_{t-1,W}+\widetilde G_{t,W}, &
N_{t,W}&=\widetilde G_{t,W}+\mu M_{t,W},\\
O_{t,W}&=\NS_J(N_{t,W}), &
W_{t+1}&=W_t-\eta O_{t,W}.
\end{aligned}
\label{eq:baseline-muon}
\end{equation}
Here $M_{t,W}$ is the outer-momentum buffer, $N_{t,W}$ is the Nesterov direction, and $O_{t,W}$ is the orthogonalized direction. The operator $\NS_J$ applies $J$ Newton--Schulz steps; Section~\ref{sec:muon} gives the exact iteration. We call $M_{t,W}$ \emph{outer momentum} to distinguish it from the per-example \emph{inner momentum} introduced next. Inner momentum acts before clipping, whereas outer momentum acts on the private release. Section~\ref{sec:method} gives the complete training procedure.

\vspace{-0.2cm}
\subsection{Metrics}
\label{sec:metrics}
\vspace{-0.2cm}
For nonzero matrices $X$ and $Y$, define the cosine error
\begin{equation}
d_{\mathrm{cos}}(X,Y)
=1-\frac{\inner{X}{Y}}{\fnorm{X}\fnorm{Y}}.
\label{eq:cosine-error}
\end{equation}
This is cosine distance after vectorizing the matrices. For a reduced singular value decomposition $X=U\Sigma V^\top$ over the positive singular values, define $\polar(X)=UV^\top$ and $\polar(0)=0$. Individual singular vectors are not unique within a repeated-singular-value subspace, but $UV^\top$ is invariant to the corresponding joint change of basis. It is the partial isometry mapping the right singular subspace to the left singular subspace. The polar error is
\begin{equation}
d_{\mathrm{polar}}(X,Y)=\fnorm{\polar(X)-\polar(Y)}.
\label{eq:polar-discrepancy}
\end{equation}
Cosine error compares all entries, whereas polar error compares this basis-invariant map used by Muon. We report both because two matrices can differ in entrywise cosine similarity while sharing a polar factor.

\vspace{-0.2cm}
\section{Geometric effects of gradient clipping}
\label{sec:clipping}
\vspace{-0.2cm}
Muon’s Newton–Schulz update preserves the input’s polar factor while transforming its singular values. This makes flat clipping, as used in standard DP-SGD, a poor fit for Muon. The reason is that different examples receive different clipping factors, clipping reweights each example's contribution. In other words, the straightforward \textit{rescale-then-average} mechanism that makes DP-SGD's privacy guarantee simple actually corrupts the geometric object that drives Muon's update.

We make this concrete in three steps. First, in \S\ref{sub:residual} we decompose the clipped batch average into a common-rescaling term, which would leave the polar factor unchanged, and a residual term $R$ that captures exactly the distortion caused by unequal per-example clipping factors. Second, in \S\ref{subsec:ns-does-not-fix} we show that Newton–Schulz orthogonalization, which computes the polar factor of a matrix, does not undo this distortion. Third, in \S\ref{subsec:example} we give a minimal two-example numerical example that makes the resulting change in update direction explicit and quantifiable.


\vspace{-0.2cm}
\subsection{The clipping residual} \label{sub:residual}
\vspace{-0.2cm}
Fix a step $t$ and a block $W\in\mathcal W$. Write $G_i$ for the matrix supplied by example $i$ for this block before clipping, and let $\lambda_i$ be the clipping factor computed from its complete query. For DP-Muon, $G_i=G_{t,W}^{(i)}$. We suppress $t$ and $W$ below.

For a batch $S_t$ of size $m$, define
\begin{equation}
G_{\mathrm{avg}}=\frac1m\sum_{i\in S_t}G_i,
\qquad
G_{\mathrm{clip}}=\frac1m\sum_{i\in S_t}\lambda_iG_i,
\qquad
\bar\lambda=\frac1m\sum_{i\in S_t}\lambda_i.
\label{eq:means}
\end{equation}
If every example has the same clipping factor, then $G_{\mathrm{clip}}=\bar\lambda G_{\mathrm{avg}}$, so clipping cannot change the singular vectors or polar factor of the mean. Unequal factors add the residual below.

Define the variation in clipping factors and matrix gradients by
\begin{equation}
\sigma_\lambda^2=\frac1m\sum_{i\in S_t}(\lambda_i-\bar\lambda)^2,
\qquad
\sigma_G^2=\frac1m\sum_{i\in S_t}\fnorm{G_i-G_{\mathrm{avg}}}^2.
\label{eq:dispersions}
\end{equation}

\begin{proposition}[Residual identity and bound]
\label{prop:residual}
The clipped mean can be written as
\begin{equation}
G_{\mathrm{clip}}=\bar\lambda G_{\mathrm{avg}}+R,
\qquad
R=\frac1m\sum_{i\in S_t}(\lambda_i-\bar\lambda)
(G_i-G_{\mathrm{avg}}),
\label{eq:residual}
\end{equation}
where the residual satisfies $\fnorm{R}\le\sigma_\lambda\sigma_G$.
\end{proposition}

The proof is in Appendix~\ref{app:residual-proof}.


The term $\bar\lambda G_{\mathrm{avg}}$ preserves the singular vectors and polar factor, so any clipping-induced change to them must arise from $R$. Not every component of $R$ changes the polar factor, however, and the bound in Proposition~\ref{prop:residual} is only a worst-case envelope; it need not predict the observed polar discrepancy tightly. The residual vanishes when the clipping factors are equal ($\sigma_\lambda=0$) or the per-example matrices are equal ($\sigma_G=0$). The same argument applies to per-matrix clipping using the block's own clipping factor. Section~\ref{sec:quality-cost} reports polar error, and Appendix~\ref{app:normalized-residual} reports the scale-normalized clipping residual for historical and current-only queries.

Appendix~\ref{app:clipping-example} gives a worked two-example calculation of the clipping factors, residual, and polar error.

\subsection{What Newton--Schulz preserves} \label{subsec:ns-does-not-fix}
\label{sec:muon}
Muon uses a small number of Newton--Schulz iterations, each a polynomial in $XX^\top$ applied to $X$:
\begin{equation}
X^{+}=aX+b(XX^\top)X+c(XX^\top)^2X,
\qquad
(a,b,c)=(3.4445,-4.7750,2.0315).
\label{eq:newton-schulz}
\end{equation}

Let $\NS_J(X)$ denote $J$ such steps, starting from $X/(\fnorm{X}+\varepsilon)$
with $\varepsilon\ge 0$.\footnote{These coefficients, from
\citet{jordan2024muon}, are chosen so that small singular values grow quickly;
the iterates approach an approximately orthogonal matrix rather than
converging exactly to the polar factor.} Observe that each step acts only on singular values: if
$X=U\Sigma V^\top$ in reduced SVD form, then
$X^+=U\,p(\Sigma)\,V^\top$, where
$
p(\sigma)=a\sigma+b\sigma^3+c\sigma^5$ acts entrywise on $\Sigma$.
The singular vectors are therefore unchanged, and the
polar factor is preserved as long as no singular value changes sign. Proposition~\ref{prop:ns} makes this precise.
\begin{proposition}[Newton--Schulz preserves the polar factor]
\label{prop:ns}
Suppose $c>0$ and $b^2<4ac$. Then, in exact arithmetic, for every nonzero
matrix $X$, every $\varepsilon\ge0$, and every finite $J\ge1$,
$\polar\bigl(\NS_J(X)\bigr)=\polar(X)$. Consequently, for nonzero $X$ and $Y$,
\begin{equation}
d_{\mathrm{polar}}\bigl(\NS_J(X),\NS_J(Y)\bigr)
= d_{\mathrm{polar}}(X,Y).
\label{eq:polar-error}
\end{equation}
Muon's coefficients satisfy both conditions: $4ac-b^2=5.189382>0$.
\end{proposition}
The proof is in Appendix~\ref{app:ns-proof}.


For a constant clipping factor across samples, $\lambda_i=\bar\lambda$, we have $R=0$ in Eq.~\eqref{eq:residual}. Then $G_{\mathrm{clip}}$ is a positive rescaling of $G_{\mathrm{avg}}$, so their polar factors, and hence their Muon updates, coincide. Any change in the polar factor due to non-constant clipping is therefore captured by the residual $R$. Note that Eq.~\eqref{eq:polar-error} concerns only the matrices entering Newton--Schulz. In DP-Muon~\citep{kim2026dpmuon}, Gaussian noise and outer momentum are applied first and may change the discrepancy. Even a common clipping rescaling can change the signal-to-noise ratio or the balance between the current gradient and the momentum buffer.





\vspace{-0.2cm}
\section{Inner momentum before clipping}
\label{sec:method}
\vspace{-0.2cm}
Section~\ref{sec:clipping} showed that unequal per-example clipping changes the clipped mean's polar factor only through the residual $R$, with $\|R\|\le\sigma_\lambda\sigma_G$ (Proposition~\ref{prop:residual}). To reduce this distortion, we average each example's gradients over the current and recent private model states \emph{before} clipping, while leaving the rest of DP-Muon unchanged. We call this \emph{inner momentum}, distinguishing it from Muon's outer momentum, which averages noisy batch gradients after clipping.

The section has three parts. \S\ref{sub:averaging} defines the per-example average. \S\ref{sub:envelope} is the central result: averaging shrinks the envelope of $R$ by a factor $\rho+(1-\rho)S_K$, where $\rho$ measures how correlated an example's gradient fluctuations are across states and $S_K$ is the sum of squared averaging weights. The factor is below one when $\rho<1$ and at least two weights are nonzero. \S\ref{sub:procedure} gives the resulting algorithm, its cost, and its privacy guarantee. The envelope result bounds $R$ and does not prove the realized $R$ is smaller, we measure the resulting polar error directly in Section~\ref{sec:experiments}.
\vspace{-0.2cm}
\subsection{Averaging one example across model states}
\label{sub:averaging}
\vspace{-0.2cm}
Let $K_t=\min\{t,K_{\max}\}$. For $i\in S_t$ and each Muon matrix block $W\in\mathcal W$, evaluate
\begin{equation}
G_{t-j,W}^{(i)}=\nabla_W\ell(\theta_{t-j};x_i),\qquad j=0,\ldots,K_t,
\label{eq:history-gradient}
\end{equation}
where $j=0$ is the current model and every term uses the same example $x_i$. With decay $\gamma\in[0,1]$ and normalized weights $\alpha_j=\gamma^j/\sum_{\ell=0}^{K_t}\gamma^\ell$, form
\begin{equation}
H_{t,W}^{(i)}=\sum_{j=0}^{K_t}\alpha_j\,G_{t-j,W}^{(i)}.
\label{eq:inner-momentum}
\end{equation}
If $K_t=0$ or $\gamma=0$, then $H_{t,W}^{(i)}=G_{t,W}^{(i)}$ and we recover DP-Muon.
For each example, the blocks $H_{t,W}^{(i)}$ are combined with the non-Muon blocks, which are averaged in the same way with their own history length and decay $(K_A,\gamma_A)$, and clipped as one query,
with a scalar factor $\lambda_i$ computed from all trainable coordinates. Proposition~\ref{prop:residual} then applies with $G_i=H_{t,W}^{(i)}$.

We normalize by $\sum_j\gamma^j$, unlike the unnormalized sum of \citet{xiao2023theory}, for two reasons. First, $H_{t,W}^{(i)}$ stays on the scale of a single gradient, so the clipping radius $C$ keeps its meaning. An unnormalized sum would inflate query norms by up to $\sum_j\gamma^j$ and force stronger clipping. Second, the weights sum to one, which is what the analysis in \S\ref{sub:envelope} uses.

\vspace{-0.2cm}
\subsection{Averaging shrinks the residual envelope}
\label{sub:envelope}
\vspace{-0.2cm}

Proposition~\ref{prop:residual} bounds $R$ by the product $\sigma_\lambda\sigma_G$, so the question is whether averaging lowers both factors. Let $g_j(z)=\nabla_\theta\ell(\theta_{t-j};z)\in\mathbb R^d$ be example $z$'s full-query gradient over all trainable coordinates, including Muon and non-Muon parameters. Define $\mu_{j,u}=\mathbb E\,g_j(z)$ and $\xi_{j,u}(z)=g_j(z)-\mu_{j,u}$. With the same weights for all coordinates, the averaged query is $u_i=\sum_j\alpha_jg_j(z_i)$. Let $S_K=\sum_j\alpha_j^2$, which is $1$ without averaging and $1/(K+1)$ for equal weights. We assume, for constants $\tau_u\ge0$ and $\rho\in[0,1]$, that $\mathbb E\|\xi_{j,u}\|_2^2\le\tau_u^2$ and $\mathbb E\langle\xi_{j,u},\xi_{\ell,u}\rangle\le\rho\tau_u^2$ for $j\neq\ell$ on the same example. The examples in a batch are independent. Set
\begin{equation}
s_{K,u}^2=\tau_u^2\bigl(\rho+(1-\rho)S_K\bigr).
\end{equation}
Here $\tau_u$ bounds the root-mean-square fluctuation of a full-query gradient, and $\rho$ bounds its cross-state covariance relative to $\tau_u^2$.


\begin{proposition}[Residual envelope under averaging]
\label{prop:residual-envelope}
Let $\lambda_i=\min\{1,C/\|u_i\|_2\}$ be the flat-clipping factor of example $i$'s averaged query $u_i$, and let $R_W$ be the residual of Proposition~\ref{prop:residual} for a block $W$ of the query. Then
\begin{equation}
\mathbb E\,\|R_W\|_F\;\le\;\frac{m-1}{m}\,\frac{s_{K,u}^2}{C}
\;=\;\frac{m-1}{m}\,\frac{\tau_u^2}{C}\,\bigl(\rho+(1-\rho)S_K\bigr).
\label{eq:residual-envelope}
\end{equation}
The current-only query ($K=0$) satisfies the same bound with $s_{0,u}=\tau_u$. Averaging therefore shrinks the envelope by the factor $\rho+(1-\rho)S_K$, which is below one whenever $\rho<1$ and at least two weights are nonzero.
\end{proposition}

\emph{Why it holds.} The proof has two steps, each answering one factor of Proposition~\ref{prop:residual}. \emph{(i) $\sigma_G$ shrinks.} The weights sum to one, so the population means $\mu_{j,u}$ cancel when one example's averaged gradient is compared with the batch mean. The remaining fluctuations combine with variance at most $\tau_u^2(\rho+(1-\rho)S_K)$.
 \emph{(ii) $\sigma_\lambda$ is controlled by the same dispersion.} The map $n\mapsto\min\{1,C/n\}$ is $1/C$-Lipschitz, and $|\,\|u_i\|-\|\bar u\|\,|\le\|u_i-\bar u\|$, so $\sigma_\lambda\le\sigma_u/C$, where $\sigma_u$ is the dispersion of the complete queries. A block is a subset of coordinates, so $\sigma_{G,W}\le\sigma_u$, and hence $\|R_W\|_F\le\sigma_\lambda\sigma_{G,W}\le\sigma_u^2/C$, which is quadratic in the dispersion.
 The full proof is in Appendix~\ref{app:polar}.

\emph{Separate histories.} Proposition~\ref{prop:residual-envelope} uses one set of weights for all coordinates. Suppose instead that each parameter group $b$ (for example, Muon and AdamW blocks) uses its own weights, with constants $\tau_b,\rho_b$ and squared-weight sum $S_{K,b}$. Set $s_b^2=\tau_b^2\bigl(\rho_b+(1-\rho_b)S_{K,b}\bigr)$ and $s_u^2=\sum_b s_b^2$. The same argument, combined with Cauchy--Schwarz, gives
\begin{equation}
\mathbb E\,\|R_W\|_F\le\frac{m-1}{m}\,\frac{s_u\,s_W}{C},
\label{eq:residual-envelope-groups}
\end{equation}
where $s_W$ is the constant of the group containing $W$. A group without history ($S_{K,b}=1$) contributes an unreduced term to $s_u$.

\emph{What it does and does not say.} It bounds $R$ from above and does not prove that the realized $R$ is smaller. No lower bound on the current-only residual is available in general. The reduction also needs $\rho<1$, meaning an example's gradient fluctuations at nearby states must not be perfectly correlated. This is plausible with dropout or other evaluation randomness and with parameter drift\rc{.}
Staleness does not enter Eq.~\eqref{eq:residual-envelope}: averaging over older states affects the residual only through fluctuations, and affects the bias of the averaged gradient through a drift term $D_K$ that grows with $K$ and $\gamma$. Appendix~\ref{app:polar} (Theorem~\ref{thm:polar}) shows how the two combine in a polar-error bound. This is the trade-off that $K_{\max}$ and $\gamma$ control.

\vspace{-0.2cm}
\subsection{Training procedure}
\label{sub:procedure}
\vspace{-0.2cm}

Algorithm~\ref{alg:dp-muon-im} gives the update for one Muon block $W\in\mathcal W$, where $\mathcal W$ is the set of Muon matrices. We drop the block subscript and apply the same steps to every block. The only change from DP-Muon is line~4, which replaces the current gradient $G_t^{(i)}$ by the averaged $H_t^{(i)}$ of Eq.~\eqref{eq:inner-momentum}. Clipping, noise, outer momentum, and Newton--Schulz are unchanged. Because Newton--Schulz preserves the polar factor (Proposition~\ref{prop:ns}), any correction must act before it, so we leave it unchanged.
\vspace{-0.3cm}
\begin{algorithm}[H]
\caption{DP-Muon with inner momentum (DP-Muon-IM).}
\label{alg:dp-muon-im}
\small
\begin{algorithmic}[1]
\REQUIRE Iterations $T$, learning rate $\eta$, momentum $\mu$, and NS steps $J$
\REQUIRE History $K_{\max},\gamma$ and privacy parameters $C,\sigma_{\mathrm{DP}}, q,B$
\STATE $M_{-1}\leftarrow0$.
\FOR{$t=0,\ldots,T-1$}
    \STATE Poisson-sample $S_t\subset\mathcal D$ with rate $q$; $K_t\leftarrow\min\{t,K_{\max}\}$.
    \STATE $\displaystyle H_t^{(i)}\leftarrow
    \frac{\sum_{j=0}^{K_t}\gamma^j\nabla_W\ell(\theta_{t-j};x_i)}
         {\sum_{j=0}^{K_t}\gamma^j},\qquad i\in S_t$.
    \STATE $\displaystyle \widetilde G_t\leftarrow
    \frac{1}{B}\left(\sum_{i\in S_t}\clip_C(H_t^{(i)})
    +\sigma_{\mathrm{DP}}C Z_t\right),\quad Z_t\sim\mathcal N(0,I_{\dim(W)})$.
    \STATE $M_t\leftarrow \mu M_{t-1}+\widetilde G_t$.
    \STATE $N_t\leftarrow \widetilde G_t+\mu M_t$.
    \STATE $O_t\leftarrow \NS_J(N_t)$.
    \STATE $W_{t+1}\leftarrow W_t-\eta O_t$.
\ENDFOR
\RETURN $\theta_T$
\end{algorithmic}
\end{algorithm}
\vspace{-0.3cm}

\emph{Cost.} The history holds at most $K_{\max}$ pre-update model states, and each sampled example needs up to $K_t+1$ gradient evaluations, so the per-step cost grows linearly in $K_{\max}$.

\emph{Privacy.} History averaging adds no private releases. To analyze privacy at each step, we hold the previous noisy releases fixed, which also fixes the current and stored model states. Each sampled example then contributes one query clipped at $C$, and Gaussian noise is added once to the batch sum. The baseline accounting applies whenever the sampling rule, clipping radius, noise calibration, release denominator, and number of releases are unchanged. Theorem~\ref{thm:dp-muon-im} makes this precise (see Appendix~\ref{app:privacy} for the proof).

\begin{theorem}[Privacy of DP-Muon-IM]
\label{thm:dp-muon-im}
Algorithm~\ref{alg:dp-muon-im} is $(\epsilon,\delta)$-DP under add/remove adjacency whenever a valid Poisson-sampled Gaussian accountant certifies this guarantee for $T$ updates with sampling rate $q$ and noise multiplier $\sigma_{\mathrm{DP}}$. The guarantee holds for any fixed public history length $K_{\max}$ and decay $\gamma\in[0,1]$.
\end{theorem}


\vspace{-0.2cm}
\section{Experiments}
\label{sec:experiments}
\vspace{-0.2cm}
We fine-tune GPT-2 on two table-to-text generation tasks and compare
DP-Muon-IM with DP-Muon across privacy budgets. We also compare with published
Muon-based private optimizers, and study how the history length affects
generation quality, training cost, and polar error.

\textbf{Data and model.}
We fully fine-tune GPT-2-small on E2E~\citep{E2E} and DART~\citep{DART}.
This is a standard benchmark for private fine-tuning
\citep{li2022large,zhang2025disk} and the setting used to evaluate DP-Muon
\citep{kim2026dpmuon}. Training/validation/test splits contain
$42{,}061/4{,}672/4{,}693$ records on E2E and $62{,}659/6{,}980/12{,}552$ on
DART.

\textbf{Optimizers and training.}
Our baseline, DP-Muon, disables both inner histories ($K=0$). DP-Muon-IM adds separate Muon and AdamW histories, with lengths $K_M,K_A$ and decays $\gamma_M,\gamma_A$, extending Section~\ref{sec:method}. Each per-example query is clipped at $C=1$ before noise is added.

We train for ten epochs on one H100 80GB, with nominal logical/physical batches of $1024/32$. Muon uses momentum $0.95$, Nesterov updates, and five Newton--Schulz steps. Learning rates decay linearly without warm-up or weight decay. Hyperparameter details for our runs are provided in Appendix~\ref{app:experiment-details}.

\textbf{Privacy and selection.}
Our utility runs target $\epsilon\in\{1,2,4,8\}$ with $\delta=1/(2N)$, where $N$ is the number of training records. We calibrate noise using PRV accounting~\citep{PRV}. We select configurations by validation loss on seed $37$ and evaluate seeds $37$--$39$.

\textbf{Evaluation.}
We compute BLEU~\citep{BLEU} and multi-reference ROUGE-L~\citep{ROUGE}. Both scores are reported on a
$0$--$100$ scale; higher is better. E2E has $630$ test prompts
and $4{,}693$ references; DART has $5{,}097$ prompt groups
and $12{,}551$ nonempty references.

\vspace{-0.2cm}
\subsection{Comparison with existing methods}
\vspace{-0.2cm}

Table~\ref{tab:utility} compares Muon-based optimizers on E2E at target $\epsilon=8$. We take the DP-Muon and DP-Muon-BC scores from \citet{kim2026dpmuon}; DP-Muon-BC adds a correction for privacy-noise bias after Newton--Schulz. Our KF variants apply DiSK's simplified Kalman filter to the private release before the Muon and AdamW updates~\citep{zhang2025disk}. KF-DP-Muon disables inner history; KF-DP-Muon-IM adds history averaging. DP-Muon and DP-Muon-IM use per-metric maxima over seeds $37$--$39$.

\begin{table}[t]
\caption{E2E test results for GPT-2 fine-tuning at target $\epsilon=8$.}
\label{tab:utility}
\centering
\small
\setlength{\tabcolsep}{6pt}
\renewcommand{\arraystretch}{1.16}
\begin{tabularx}{0.8\linewidth}{@{}Xlrr@{}}
\toprule
\textbf{Method} & \textbf{Source} & \textbf{BLEU} $\uparrow$ & \textbf{ROUGE-L} $\uparrow$ \\
\midrule

DP-Muon & \citet{kim2026dpmuon} & $63.78$ & $66.96$\\
DP-Muon-BC & \citet{kim2026dpmuon} & $63.66$ & $66.71$\\
\midrule
KF-DP-Muon-IM & This work & $64.77$ & $67.37$\\

KF-DP-Muon & This work & $64.95$ & $67.36$\\
DP-Muon & This work & $64.79$ & $67.58$\\
\rowcolor{black!5} \textbf{DP-Muon-IM} & This work & $65.65$ & $67.80$\\
\bottomrule
\end{tabularx}
\end{table}

\vspace{-0.2cm}
\subsection{Utility across privacy budgets}
\vspace{-0.2cm}

DP-Muon-IM improves BLEU and ROUGE-L over DP-Muon in all $24$ seed-matched comparisons (two datasets, four privacy budgets, three seeds; Table~\ref{tab:im-e2e-dart-privacy}). Mean BLEU gains are $0.60$--$2.57$ points on E2E and $0.50$--$1.06$ on DART. Selected E2E configurations use $K_M=4$, $\gamma_M=\gamma_A=0.5$, and $K_A=3$ for $\epsilon\in\{1,2,4\}$ or $K_A=2$ for $\epsilon=8$. DART uses $(K_M,K_A)=(3,3)$ and $(\gamma_M,\gamma_A)=(0.7,0.5)$ throughout.

Appendix~\ref{app:group-im} reports a fixed-learning-rate ablation with Muon-only history, AdamW-only history, both, and neither.

\begin{table}[t]
\caption{Test generation after ten epochs: mean $\pm$ standard deviation over seeds $37$--$39$.}
\label{tab:im-e2e-dart-privacy}
\centering
\small
\setlength{\tabcolsep}{5pt}
\renewcommand{\arraystretch}{1.08}
\begin{tabular}{@{}lcrrrr@{}}
\toprule
& & \multicolumn{2}{c}{DP-Muon} & \multicolumn{2}{c}{DP-Muon-IM} \\
\cmidrule(lr){3-4}\cmidrule(l){5-6}
Task & Target $\epsilon$ & BLEU $\uparrow$ & ROUGE-L $\uparrow$ & BLEU $\uparrow$ & ROUGE-L $\uparrow$ \\
\midrule
E2E & $1$ & $57.64\pm1.97$ & $62.71\pm0.21$ & $60.21\pm2.65$ & $64.41\pm1.02$ \\
    & $2$ & $63.05\pm0.34$ & $65.89\pm0.62$ & $63.76\pm0.26$ & $66.43\pm0.60$ \\
    & $4$ & $64.10\pm0.29$ & $66.72\pm0.18$ & $64.83\pm0.57$ & $67.03\pm0.39$ \\

    & $8$ & $64.74\pm0.05$ & $67.31\pm0.25$ & $65.34\pm0.39$ & $67.60\pm0.17$ \\

\midrule
DART & $1$ & $32.34\pm0.59$ & $50.20\pm0.23$ & $33.19\pm0.42$ & $50.91\pm0.16$ \\
     & $2$ & $34.88\pm0.37$ & $52.27\pm0.12$ & $35.94\pm0.29$ & $53.09\pm0.13$ \\
     & $4$ & $36.48\pm0.35$ & $53.52\pm0.27$ & $37.16\pm0.33$ & $54.24\pm0.17$ \\
     & $8$ & $37.44\pm0.33$ & $54.35\pm0.30$ & $37.95\pm0.19$ & $54.79\pm0.12$ \\
\bottomrule
\end{tabular}
\end{table}

\vspace{-0.2cm}
\subsection{Effect of history length}
\label{sec:quality-cost}
\vspace{-0.2cm}

\textbf{Utility and cost.}
We vary $K_M=K_A=K\in\{0,\ldots,4\}$ on E2E at $\epsilon\approx4$ for ten epochs, with Muon/AdamW learning rates $0.002/0.0035$ and decays $\gamma_M=\gamma_A=0.5$ for $K>0$; $K=0$ is DP-Muon without history. One historical state provides most of the BLEU gain: $K=0,1,4$ yield validation BLEU scores of $66.97$, $67.76$, and $67.82$ in $36.9$, $59.0$, and $123.7$ minutes, respectively, with nonmonotonic gains across $K$ (Figure~\ref{fig:e2e-final-tradeoff}, left). Peak GPU allocation rises from $26.70$ GiB without history to $49.85$ GiB with history. Timings exclude setup, evaluation, and checkpoint I/O.

\begin{figure}[htbp]
\noindent
\begin{minipage}[t]{0.482\linewidth}
\vspace{-1cm}
\includegraphics[width=\linewidth]{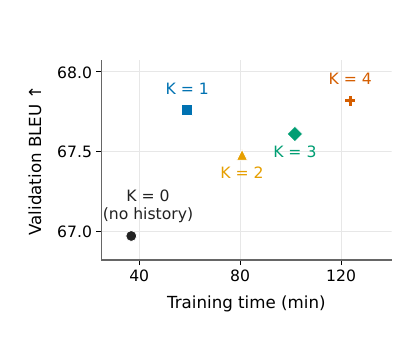}
\end{minipage}\hfill
\begin{minipage}[t]{0.482\linewidth}
\vspace{-1cm}
\includegraphics[width=\linewidth]{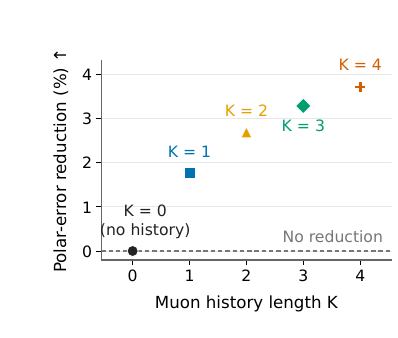}
\end{minipage}
 \vspace{-0.3cm}
\caption{Effect of history length in separate utility and geometry experiments. \textbf{Left:} validation BLEU versus training time after ten epochs. \textbf{Right:} pre-noise polar-error reduction before optimizer update step $10$, with Muon-only history.}
\label{fig:e2e-final-tradeoff}
 \vspace{-0.3cm}
\end{figure}

\textbf{Polar error.}
We measure polar error with Muon-only history, $\gamma=0.9$ (Figure~\ref{fig:e2e-final-tradeoff}, right). Each $K$ follows a separate trajectory. Within each trajectory, we compare IM and current-only clipped means at the same model state against a shared unclipped current-gradient reference, before noise and outer momentum. Error is the root-sum-square polar-factor distance across $48$ Muon matrices, using $512$ query records and $2{,}048$ disjoint held-out reference records. Just before the 10th optimizer update, IM reduces polar error by $1.76\%$ at $K=1$ and $3.71\%$ at $K=4$ relative to the current-only estimate.

\textbf{Clipping residual.}
With $K=4$ and $\gamma=0.9$, historical averaging lowers the normalized clipping residual in all $21$ seed--checkpoint comparisons. Mean paired reductions range from $25.0\%$ to $46.1\%$ across checkpoints (Appendix~\ref{app:normalized-residual}). These raw-query diagnostics precede noise and outer momentum.

\vspace{-0.2cm}
\section{Related Work}
\label{sec:related-work}
\vspace{-0.2cm}
\textbf{Clipping bias and pre-clipping averaging.}
Geometric analyses show that clipping can bias private optimization by changing the aggregate gradient~\citep{chen2020geometric}. \citet{xiao2023theory} relate this bias to per-example sampling noise, introduce inner--outer momentum, and report cosine similarity between unclipped and averaged-clipped gradients. Adaptive clipping and error feedback offer other treatments of clipping bias~\citep{bu2022automatic,zhang2024error}; DP-PMLF combines per-sample momentum with post-release filtering~\citep{xu2026pmlf}. We adapt pre-clipping averaging to Muon's matrix updates and study the resulting clipping residual (Section~\ref{sub:residual}).

\textbf{Muon under differential privacy.}
Muon applies momentum and a finite Newton--Schulz polynomial to hidden-layer matrix gradients~\citep{jordan2024muon}. DP-Muon studies private matrix optimization and Gaussian perturbations under this transform~\citep{kim2026dpmuon}; \citet{kimoh2026ns} analyze optimization with finitely many Newton--Schulz steps. Section~\ref{subsec:ns-does-not-fix} characterizes the polar geometry preserved by the iteration used here.


\vspace{-0.2cm}
\section{Discussion}
\label{sec:discussion}
\vspace{-0.2cm}
The E2E and DART results support averaging per-example gradients before clipping. The residual decomposition explains why this stage matters: unequal clipping can change the batch mean's polar factor, while the Newton--Schulz iteration used here preserves its immediate input's polar factor in exact arithmetic.

The geometry probes compare clipped IM and current-only means against the same current-gradient reference. With history length $K=4$ and decay $\gamma=0.9$, IM reduces polar error by $2.4$--$3.7\%$ across measured checkpoints (Appendix~\ref{app:reorientation-training}). These measurements precede Gaussian noise and outer momentum; they do not measure the final Muon update.

The E2E ablation favors a short history when compute is limited: $K=1$ came within $0.06$ validation BLEU points of $K=4$ in less than half the training time (Figure~\ref{fig:e2e-final-tradeoff}). This result favors a short history in the tested setting. The extra gradient evaluations and stored model states remain costs to address in larger-scale experiments.

\vspace{-0.2cm}
\section{Conclusion}
\label{sec:conclusion}
\vspace{-0.2cm}
We identify a covariance residual caused by unequal clipping and show that the Newton--Schulz iteration used here preserves its input's polar factor in exact arithmetic. DP-Muon-IM averages per-example gradients at current and recent private model states before clipping, retaining baseline per-run privacy accounting under the stated conditions. In GPT-2 fine-tuning, DP-Muon-IM achieves higher BLEU and ROUGE-L than DP-Muon on E2E and DART at target $\epsilon\in\{1,2,4,8\}$.

\subsection*{AI Use Statement}
We used large language models for language editing, LaTeX support, bibliographic checks, and mathematical consistency. We reviewed all AI-assisted work and take full responsibility for the paper.

\bibliographystyle{plainnat}
\bibliography{references_clean}

\appendix
\newpage

\section{Proofs}
\subsection{Residual identity and bound}
\label{app:residual-proof}

We use the means and dispersions defined in Eqs.~\eqref{eq:means} and~\eqref{eq:dispersions}.

\begin{proof}[Proof of Proposition~\ref{prop:residual}]
Expanding $R$ and using $\sum_{i\in S_t}(\lambda_i-\bar\lambda)=0$ gives
\[
R=\frac1m\sum_{i\in S_t}\lambda_iG_i
-\bar\lambda G_{\mathrm{avg}}
=G_{\mathrm{clip}}-\bar\lambda G_{\mathrm{avg}}.
\]
The triangle inequality and Cauchy--Schwarz yield
\[
\fnorm{R}
\le \frac1m\sum_{i\in S_t}|\lambda_i-\bar\lambda|\,
\fnorm{G_i-G_{\mathrm{avg}}}
\le \sigma_\lambda\sigma_G. \qedhere
\]
\end{proof}

\subsection{Newton--Schulz polar-factor preservation}
\label{app:ns-proof}

Recall $p(\sigma)=a\sigma+b\sigma^3+c\sigma^5$, the scalar polynomial associated with Eq.~\eqref{eq:newton-schulz}.

\begin{proof}[Proof of Proposition~\ref{prop:ns}]
Completing the square gives $p(\sigma)=\sigma\bigl[c(\sigma^2+b/2c)^2+(4ac-b^2)/4c\bigr]>0$ for all $\sigma>0$. Starting from a reduced SVD $X=U\Sigma V^\top$, the normalization and each step replace $\Sigma$ by another positive diagonal matrix while keeping $U$ and $V$ fixed. The new diagonal entries need not be in decreasing order, but sorting them only permutes the columns of $U$ and $V$ jointly, which leaves $UV^\top$ unchanged. Hence every iterate has a reduced SVD with polar factor $UV^\top$, and by the definition in Section~\ref{sec:metrics}, $\polar(\NS_J(X))=\polar(X)$.
\end{proof}

\section{Illustrative example of clipping distortion}
\label{app:clipping-example}
\label{subsec:example}
\label{sec:motivating-example}

Consider two per-example gradient matrices:
\[
G_1=\begin{bmatrix}10&0\\0&10\end{bmatrix},
\qquad
G_2=\begin{bmatrix}0&-1\\1&0\end{bmatrix}.
\]
With clipping radius $C=\sqrt2$, their clipping factors are $0.1$ and $1$. The means before and after clipping are
\[
G_{\mathrm{avg}}=\begin{bmatrix}5&-0.5\\0.5&5\end{bmatrix},
\qquad
G_{\mathrm{clip}}=\begin{bmatrix}0.5&-0.5\\0.5&0.5\end{bmatrix}.
\]
The mean clipping factor is $\bar\lambda=0.55$, so the residual is
\[
R=G_{\mathrm{clip}}-0.55G_{\mathrm{avg}}
=\begin{bmatrix}-2.25&-0.225\\0.225&-2.25\end{bmatrix}.
\]
Here $\sigma_\lambda=0.45$ and $\sigma_G=\sqrt{202}/2$, giving
\[
\fnorm{R}=0.225\sqrt{202}
=\sigma_\lambda\sigma_G\approx3.198.
\]
The bound is attained in this two-example construction, but it need not be tight in general. More importantly, $G_{\mathrm{clip}}$ is not a scalar multiple of $G_{\mathrm{avg}}$: clipping changes the aggregate polar factor even though it preserves the polar factor of either example.

In the metrics of Section~\ref{sec:metrics}, the resulting errors are
\[
d_{\mathrm{cos}}(G_{\mathrm{clip}},G_{\mathrm{avg}})
=1-\frac{11}{\sqrt{202}}\approx0.226,
\qquad
d_{\mathrm{polar}}(G_{\mathrm{clip}},G_{\mathrm{avg}})
=2\sqrt{1-\frac{11}{\sqrt{202}}}\approx0.951.
\]
In this two-dimensional special case, the polar factors happen to be planar rotations by $5.7^\circ$ and $45^\circ$. In higher dimensions, the polar factor instead describes a partial isometry between the right and left singular subspaces.

\section{Utility experiment details}
\label{app:experiment-details}

\textbf{Parameter groups and updates.}
Muon updates tensors with at least two dimensions, excluding token and position embeddings and the language-model head. AdamW updates those excluded parameters, vectors, and biases. For an $r\times c$ Muon matrix, the learning rate is the selected base rate times $0.2\sqrt{\max(r,c)}$~\citep{liu2025muon}. The five Newton--Schulz steps use coefficients $(3.4445,-4.7750,2.0315)$, with no extra output-level bias correction. AdamW retains standard moment bias correction and numerical $\varepsilon=10^{-8}$. Training dropout is active; historical evaluations use a separate stochastic stream, and evaluation disables dropout.

\textbf{Selected histories and learning rates.}
Table~\ref{tab:selected-rates} gives the base learning rates for the utility runs. At $\epsilon\in\{1,2,4\}$, E2E uses $(K_M,\gamma_M)=(4,0.5)$ and $(K_A,\gamma_A)=(3,0.5)$. DART uses $(K_M,\gamma_M)=(3,0.7)$ and $(K_A,\gamma_A)=(3,0.5)$ at every budget. DP-Muon (No-IM) disables both histories. Both methods in this table use AdamW $\beta_1=0.9$.
The E2E $\epsilon=8$ campaign fixes $(K_M,K_A)=(4,2)$ and $\gamma_M=\gamma_A=0.5$ for IM; its baseline disables both histories. Both use Muon/AdamW learning rates $0.002/0.0035$, retained from the earlier validation search without retuning.

At E2E $\epsilon=8$, both methods are trained afresh at $\delta=1/84122$. Across seeds $37$--$39$, IM gives $65.34\pm0.39$ BLEU and $67.60\pm0.17$ ROUGE-L, versus $64.74\pm0.05$ and $67.31\pm0.25$ for the baseline. Table~\ref{tab:utility} reports per-metric maxima: $65.65/67.80$ for IM and $64.79/67.58$ for the baseline, all from seed $39$. Seed $37$ informed the earlier search; seeds $38,39$ retain the previously used confirmation identities.

\begin{table}[htbp]
\caption{Validation-selected base learning rates for DP-Muon (No-IM) and DP-Muon-IM; each pair is (Muon, AdamW). Both methods retain the same standard optimizer moments.}
\label{tab:selected-rates}
\centering
\small
\begin{tabular}{@{}lccc@{}}
\toprule
Task & Target $\epsilon$ & DP-Muon (No-IM) & DP-Muon-IM \\
\midrule
E2E & $1$ & $(0.0016,0.0028)$ & $(0.0016,0.0035)$ \\
    & $2$ & $(0.0016,0.0035)$ & $(0.0016,0.0035)$ \\
    & $4$ & $(0.0016,0.0035)$ & $(0.0020,0.0035)$ \\
    & $8$ & $(0.0020,0.0035)$ & $(0.0020,0.0035)$ \\
\midrule
DART & $1$ & $(0.00105,0.00168)$ & $(0.00105,0.00168)$ \\
     & $2$ & $(0.0013125,0.00168)$ & $(0.0013125,0.0021)$ \\
     & $4$ & $(0.0013125,0.0021)$ & $(0.0013125,0.0021)$ \\
     & $8$ & $(0.0013125,0.0021)$ & $(0.0013125,0.002688)$ \\
\bottomrule
\end{tabular}
\end{table}

\textbf{Private release.}
E2E uses Poisson sampling with $q=1/41$ for $410$ releases; DART uses $q=1/61$ for $610$ releases. The implementation uses $\lfloor B\rfloor$ rather than $B$ in Eq.~\eqref{eq:private-release}, giving release denominators $1025$ and $1027$, respectively. These differ from the nominal logical batch size $1024$. Clipping uses the full concatenated per-example query, and each logical update makes one Gaussian release, not one release per microbatch. Table~\ref{tab:release-calibration} lists the verified per-run calibration.

\begin{table}[htbp]
\caption{Per-run PRV calibration for our utility runs, with $\delta=1/(2N)$ at every budget. Achieved $\epsilon$ is rounded.}
\label{tab:release-calibration}
\centering
\small
\begin{tabular}{@{}lcrrc@{}}
\toprule
Task & Target $\epsilon$ & Noise multiplier $\sigma_{\mathrm{DP}}$ & Achieved $\epsilon$ & $\delta$ \\
\midrule
E2E & $1$ & $2.05078125$ & $0.996290$ & $1/84122$ \\
    & $2$ & $1.2646484375$ & $1.999911$ & $1/84122$ \\
    & $4$ & $0.889892578125$ & $3.996703$ & $1/84122$ \\
    & $8$ & $0.67626953125$ & $7.998976$ & $1/84122$ \\
\midrule
DART & $1$ & $1.767578125$ & $0.993732$ & $1/125318$ \\
     & $2$ & $1.123046875$ & $1.998849$ & $1/125318$ \\
     & $4$ & $0.8197021484375$ & $3.993658$ & $1/125318$ \\
     & $8$ & $0.63720703125$ & $7.995195$ & $1/125318$ \\
\bottomrule
\end{tabular}
\end{table}

\textbf{DiSK settings.}
Each DiSK closure forms its temporal query. DiSK combines the base and extrapolated queries before clipping and noise, then filters the private release before Muon/AdamW. The filter uses $\kappa$ times the previous state plus $(1-\kappa)$ times the new input.

The single-seed KF rows in Table~\ref{tab:utility} form a separate joint ablation using learning rates $0.002/0.0021$, $\kappa=0.5$, and DiSK extrapolation coefficient $0.8$. The IM row retains the same histories. The IM setting was selected by validation loss from $30$ new trials and one reused anchor. Both KF studies use $\sigma_{\mathrm{DP}}=0.67626953125$, $\delta=1/84122$, and achieved $\epsilon\approx7.998976$. DiSK's extrapolation coefficient is distinct from either history decay.

\subsection{Optimizer-group ablation}
\label{app:group-im}

We switch history on or off separately for Muon and AdamW, fixing their base learning rates at $0.002$ and $0.0035$. Enabled histories use $(K_M,\gamma_M)=(4,0.5)$ and $(K_A,\gamma_A)=(3,0.5)$. All four arms retain Muon momentum $0.95$, AdamW $\beta_1=0.9$ and the E2E $\epsilon=4$ privacy calibration in Table~\ref{tab:release-calibration}. 

\begin{table}[htbp]
\caption{Optimizer-group ablation on E2E at target $\epsilon=4$ and $\delta=1/84122$. NLL is measured in nats per target token.}
\label{tab:group-im-seed38}
\centering
\small
\setlength{\tabcolsep}{6pt}
\begin{tabular}{@{}lrrr@{}}
\toprule
History enabled & Test NLL $\downarrow$ & BLEU $\uparrow$ & ROUGE-L $\uparrow$ \\
\midrule
Neither & $1.2842$ & $63.66$ & $66.90$ \\
Muon only & $1.2779$ & $63.76$ & $66.95$ \\
AdamW only & $1.2686$ & $64.37$ & $67.12$ \\
Both & $1.2586$ & $65.30$ & $67.47$ \\
\bottomrule
\end{tabular}
\end{table}

\section{Geometric error across training}
\label{app:reorientation-training}

The corrected diagnostic uses consistent per-record mean target-token loss for current and historical gradients. It compares IM with a current-only query at the same state of each IM-trained trajectory, before noise and outer momentum. Both are compared with the shared finite raw-current reference described in Section~\ref{sec:quality-cost}. This is not a comparison of independently trained IM and baseline models.

Runs use E2E, GPT-2, batch size $8$, learning rate $5\times10^{-5}$, Muon $(K,\gamma)=(4,0.9)$, and no AdamW history. Full-query clipping uses $C=1$. The three training seeds are $42,43,44$; probes occur before updates $10,50,100,250,500,1000,2000$. Thus index $10$ means nine completed updates. Means accumulate in FP32 and polar factors use CPU FP64 SVD.

\begin{figure}[H]
    \centering
    \includegraphics[width=0.58\linewidth]{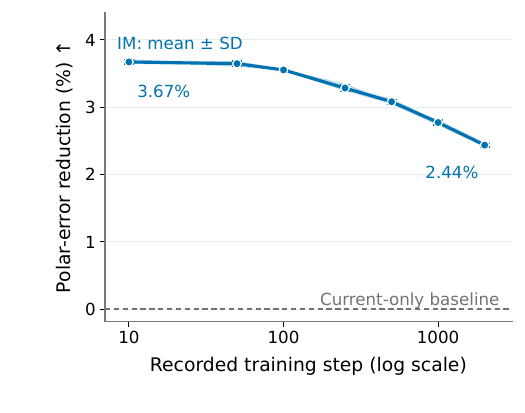}
    \caption{Corrected three-seed polar diagnostic at $(K,\gamma)=(4,0.9)$. Relative reduction is computed within each seed as $100(1-E_{\mathrm{IM}}/E_{\mathrm{current}})$, then averaged. Faint lines show individual seeds; bars and shading are pointwise sample SD, not confidence intervals.}
    \label{fig:reorientation-training}
\end{figure}

Both cosine and polar error are lower with IM in all $21$ seed--checkpoint pairs. Mean polar-error reduction decreases from $3.67\%$ before update $10$ to $2.44\%$ before update $2000$ (Figure~\ref{fig:reorientation-training}); mean cosine-error reductions range from $3.20\%$ to $5.98\%$. These distances are relative to a finite reference, not a population gradient or the noisy Newton--Schulz update.

The one-seed $K$ sweep in Figure~\ref{fig:e2e-final-tradeoff} uses the same raw-query diagnostic design at index $10$. It pairs IM and current-only queries within each trajectory; the $K=4$ point reuses the corrected seed-$42$ run above. The $K=0$--$3$ points are completed early probes, not final-training results. The two panels of that figure therefore differ in history groups, normalization, seed, batch size, and training horizon.

\subsection{Normalized clipping residual}
\label{app:normalized-residual}

We measure the clipping residual on the same $m=512$ fixed held-out query records at each checkpoint. Let $H_W^{(i)}$ be record $i$'s query for Muon block $W$: a history average for IM and a single current-state evaluation for the control. Write $G_{\mathrm{avg},W}=m^{-1}\sum_i H_W^{(i)}$ and $G_{\mathrm{clip},W}=m^{-1}\sum_i\lambda_i H_W^{(i)}$. Clipping factors come from the complete raw query at $C=1$. We report
\begin{equation}
R_W=G_{\mathrm{clip},W}-\bar\lambda G_{\mathrm{avg},W},
\qquad
\eta_R=
\frac{\bigl(\sum_{W\in\mathcal W}\fnorm{R_W}^2\bigr)^{1/2}}
{\bar\lambda\bigl(\sum_{W\in\mathcal W}\fnorm{G_{\mathrm{avg},W}}^2\bigr)^{1/2}}.
\label{eq:normalized-residual-diagnostic}
\end{equation}
The norms aggregate the same $48$ Muon matrices. Each arm uses its own mean clipping factor $\bar\lambda$ and unclipped query mean in the denominator. These measurements precede Gaussian noise and outer momentum.

\begin{figure}[H]
    \centering
    \includegraphics[width=0.9\linewidth]{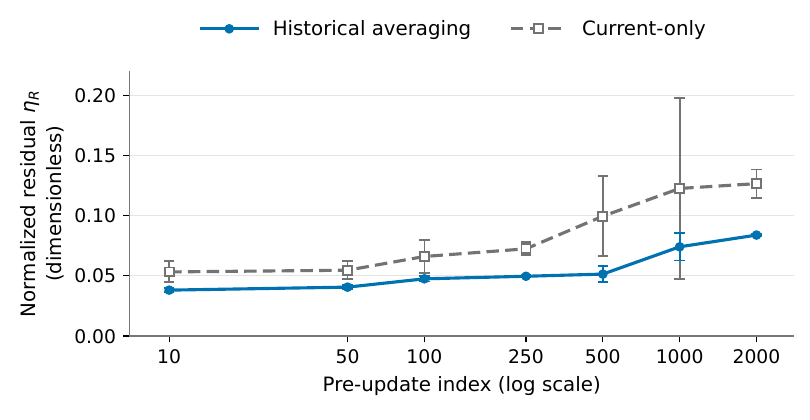}
    \caption{Normalized clipping residual $\eta_R$ on E2E for historical averaging ($K=4$, $\gamma=0.9$) and a single current-state evaluation at the same model states. Points show means over training seeds $42$--$44$; bars show sample standard deviation, not confidence intervals. All seven pre-update checkpoints are shown. Clipping uses the complete raw query at $C=1$; the residual metric covers Muon coordinates.}
    \label{fig:normalized-residual}
\end{figure}

Historical averaging gives a smaller $\eta_R$ in all $21$ seed--checkpoint pairs (Figure~\ref{fig:normalized-residual}). We compute the percentage reduction within each seed as $100(1-\eta_{R,\mathrm{IM}}/\eta_{R,\mathrm{current}})$, then average over seeds. These means range from $24.99\%$ to $46.10\%$ across checkpoints.

\section{Privacy and implementation conditions}
\label{app:privacy}

Fix a public initialization and public values of the sampling rate $q\in(0,1]$, clipping radius $C>0$, noise multiplier $\sigma_{\mathrm{DP}}>0$, release denominator $b>0$, and update count $T$. Current and stored model states and optimizer states must be computed only from preceding noisy releases and public settings. At each update, sample every record independently with probability $q$, independently of previous sampling draws. Sampling decisions remain internal. The accountant must bound adaptive composition under add/remove adjacency.

\textbf{Complete-query mechanism.}
Write $\mathcal T_{t-1}=(\widetilde G_0,\ldots,\widetilde G_{t-1})$ for the preceding noisy releases, with $\mathcal T_{-1}$ empty. Given these releases, let $u_t(x;\mathcal T_{t-1},\omega)\in\mathbb R^d$ be record $x$'s query over all trainable coordinates. Here $\omega$ is fresh data-independent randomness, drawn independently for each record and update and independently of sampling, Gaussian noise, and preceding releases. Historical evaluations within one record's query may share randomness. The query may combine different Muon and AdamW histories and apply per-record preprocessing. It must not use unaccounted private statistics from other records. Define
\begin{equation}
v_t(x;\mathcal T_{t-1},\omega)
=\clip_C\bigl(u_t(x;\mathcal T_{t-1},\omega)\bigr),
\qquad \|v_t(x;\mathcal T_{t-1},\omega)\|_2\le C.
\label{eq:privacy-record-query}
\end{equation}
The complete release at step $t$ is
\begin{equation}
\widetilde G_t=\frac1b\left(
\sum_{x\in S_t}v_t(x;\mathcal T_{t-1},\omega_{t,x})
+\sigma_{\mathrm{DP}}C Z_t\right),
\qquad Z_t\sim\mathcal N(0,I_d),
\label{eq:privacy-complete-release}
\end{equation}
where $b>0$ is fixed and public. Gaussian draws are independent across updates and independent of sampling and query randomness. Algorithm~\ref{alg:dp-muon-im} displays one Muon block of this release with $b=B$; the utility implementation uses $b=\lfloor B\rfloor$. Neither choice changes the privacy bound when the denominator is held fixed across neighboring datasets. Independent Gaussian draws for individual coordinates or blocks implement the same joint release; they are not separately clipped mechanisms.

\begin{proof}[Privacy argument]
Consider neighboring datasets $\mathcal D$ and $\mathcal D'=\mathcal D\cup\{x_\star\}$. Fix a common value of $\mathcal T_{t-1}$. With the public initialization and settings fixed, the current model and stored histories are the same in both conditional executions. This does not assert that their unconditional training trajectories coincide.

For every common record, the conditional query law in Eq.~\eqref{eq:privacy-record-query} is the same on both datasets. Couple the sampling and local randomness of these common records. The additional record is selected independently with probability $q$ and contributes a vector of norm at most $C$. Thus the conditional update is a Poisson-sampled Gaussian mechanism with noise multiplier $\sigma_{\mathrm{DP}}$: the add/remove sensitivity of the unsampled sum is at most $C$, irrespective of how many historical gradients were used within each query. Division by public $b$ is post-processing.

This per-step privacy bound holds uniformly over preceding transcripts. Adaptive composition gives the complete transcript $(\widetilde G_0,\ldots,\widetilde G_{T-1})$ the $(\epsilon,\delta)$ guarantee certified by the accountant~\citep{abadi2016deep,PRV}. Updating the history, applying Muon momentum and Newton--Schulz, and applying AdamW are post-processing of these releases. The final model and any checkpoints computed solely from the transcript inherit the same guarantee~\citep{dwork2014algorithmic}.
\end{proof}

\textbf{Calibration and output scope.}
The privacy analysis holds calibration fixed across neighboring datasets. Any reference dataset size used to choose $q$, $T$, $\delta$, or $b$ must therefore be a public design value, not recomputed from the private dataset's cardinality. The guarantee uses add/remove adjacency; replace-one adjacency has an unsampled sensitivity bound of $2C$ and requires the corresponding accounting.

This is a per-run guarantee for fixed public settings, not a privacy bound for the full experimental search. Private validation, hyperparameter selection, additional training runs, or a privately trained initialization require separate accounting. Releasing raw gradients, sampling decisions, or additional data-dependent diagnostics is not covered by this guarantee. In particular, the geometric probes' diagnostic releases and non-Poisson training protocol are not covered by the utility runs' reported budgets. This privacy argument requires no smoothness, gradient-decorrelation, or polar-rank assumptions.

\section{Averaging and polar error}
\label{app:polar}

Proposition~\ref{prop:residual} in Section~\ref{sec:clipping} separates the clipped mean into a common rescaling and a residual caused by unequal clipping. Here we use this decomposition to bound polar error relative to the current population gradient, accounting for sampling error and the use of historical gradients. The analysis concerns the clipped mean before Gaussian noise and outer momentum, not the final Muon update.

\textbf{Setup.}
Fix model states $\theta_0,\ldots,\theta_K$ and one Muon block $W\in\mathbb R^{r\times c}$, where $\theta_0$ is the current state. For an example $z$, let $G_j(z)$ be its raw gradient for this block at state $j$, and define $\mu_j=\mathbb E G_j(z)$ and $\xi_j(z)=G_j(z)-\mu_j$. Expectations include any local evaluation randomness, with the same conditional law given $z$ at every state. Draw $m\ge2$ examples independently and reuse each example across states. Each example and all its evaluation draws form one independent, identically distributed tuple; the states and weights are fixed independently of these tuples.

With fixed weights $\alpha_j\ge0$, $\sum_{j=0}^{K}\alpha_j=1$, shared across the block's coordinates, set
\begin{equation}
H^{(i)}=\sum_{j=0}^{K}\alpha_jG_j(z_i),
\qquad
N=\frac1m\sum_{i=1}^{m}\lambda_i H^{(i)},
\qquad
S_K=\sum_{j=0}^{K}\alpha_j^2.
\label{eq:polar-history-setup}
\end{equation}
As in Section~\ref{sec:clipping}, $\lambda_i>0$ is computed by flat clipping of example $i$'s complete raw query, not by clipping this block separately. Write $\bar\lambda=m^{-1}\sum_i\lambda_i$ and $\sigma_\lambda^2=m^{-1}\sum_i(\lambda_i-\bar\lambda)^2$. The current-only query has $K=0$ and $S_0=1$; equal weights give $S_K=1/(K+1)$.

We assume, for fixed constants $\tau_W\ge0$, $\rho\in[0,1]$, and $c_\lambda\ge0$,
\begin{enumerate}
    \item $\mathbb E\fnorm{\xi_j}^2\le\tau_W^2$ for every $j$;
    \label{ass:polar-variance}
    \item $\mathbb E\inner{\xi_j}{\xi_\ell}\le\rho\tau_W^2$ for $j\ne\ell$, evaluated on the same example;
    \label{ass:polar-correlation}
    \item $\sigma_\lambda\le c_\lambda\bar\lambda$ almost surely.
    \label{ass:polar-clipping}
\end{enumerate}
Here $\tau_W$ bounds fluctuations of block $W$, whereas $\tau_u$ in Section~\ref{sub:envelope} bounds fluctuations of the full query. The blockwise covariance bound $\rho$ here need not equal the full-query value used there. Assumption~(\ref{ass:polar-variance}) is the blockwise version of Assumption~3.1 in \citet{xiao2023theory}. Assumption~(\ref{ass:polar-correlation}) allows temporal correlation: independent centered fluctuations satisfy $\rho=0$, but independence is not required. The choice $\rho=1$ follows from (\ref{ass:polar-variance}) by Cauchy--Schwarz and does not require identical fluctuations. Assumption~(\ref{ass:polar-clipping}) and the full-rank reference condition below are additional assumptions, not results of \citet{xiao2023theory}.

\textbf{Polar perturbation.}
For square nonsingular inputs, the following inequality follows from classical polar-factor perturbation bounds~\citep{li1995polar}. The proof below also covers rectangular matrices and rank-deficient $X$ using the partial-polar convention in Section~\ref{sec:metrics}.

\begin{lemma}[Polar perturbation with a full-rank reference]
\label{lem:polar-lipschitz}
Let $T,X\in\mathbb R^{r\times c}$ with $\operatorname{rank}(T)=\min(r,c)$. Then
\begin{equation}
d_{\mathrm{polar}}(X,T)
\le\frac{2\fnorm{X-T}}{\sigma_{\min}(T)}.
\label{eq:polar-reference-bound}
\end{equation}
Here $\sigma_{\min}(T)>0$ is the smallest of the $\min(r,c)$ singular values of $T$.
\end{lemma}

\begin{proof}
First suppose $r\ge c$. Let $Q=\polar(T)$, $U=\polar(X)$, and write $T=QP$ with $P=(T^\top T)^{1/2}\succeq sI_c$, where $s=\sigma_{\min}(T)>0$. Since $Q^\top Q=I_c$ and $U^\top U\preceq I_c$,
\[
A=I_c-\tfrac12(Q^\top U+U^\top Q)\succeq0,
\qquad
\fnorm{Q-U}^2\le2\operatorname{tr}(A).
\]
The polar factor $U$, including when $X$ is rank deficient, maximizes $\inner{X}{V}$ over $\|V\|_{\mathrm{op}}\le1$. As $Q$ is feasible, $\inner{X}{Q-U}\le0$. Hence
\[
\frac{s}{2}\fnorm{Q-U}^2
\le\operatorname{tr}(PA)
=\inner{T}{Q-U}
\le\inner{T-X}{Q-U}
\le\fnorm{T-X}\fnorm{Q-U}.
\]
Dividing by $\fnorm{Q-U}$ proves the claim for $r\ge c$ when it is nonzero; the zero case is immediate. For $r<c$, apply this result to $T^\top$ and $X^\top$. Transposition preserves Frobenius norms and singular values, and $\polar(X^\top)=\polar(X)^\top$.
\end{proof}

\begin{theorem}[Polar-error bound under averaging]
\label{thm:polar}
Under Assumptions (\ref{ass:polar-variance})--(\ref{ass:polar-clipping}), suppose $\mu_0\in\mathbb R^{r\times c}$ has rank $\min(r,c)$. Define
\begin{equation}
s_{K,W}=\tau_W\sqrt{\rho+(1-\rho)S_K},
\qquad
D_K=\fnorm{\sum_{j=0}^{K}\alpha_j\mu_j-\mu_0}.
\label{eq:polar-history-terms}
\end{equation}
Then
\begin{equation}
\mathbb E\,d_{\mathrm{polar}}(N,\mu_0)
\le\frac{2}{\sigma_{\min}(\mu_0)}
\left(D_K+\frac{s_{K,W}}{\sqrt m}
+c_\lambda\sqrt{\frac{m-1}{m}}\,s_{K,W}\right).
\label{eq:polar-history-bound}
\end{equation}
The three terms account for staleness, sampling error, and unequal clipping, respectively. No rank condition is imposed on $N$.
\end{theorem}

\begin{proof}
Let $Y_i=\sum_j\alpha_j\xi_j(z_i)$ and $\bar Y=m^{-1}\sum_iY_i$. Applying Proposition~\ref{prop:residual} with $G_i=H^{(i)}$ gives
\begin{equation}
\frac{N}{\bar\lambda}-\mu_0
=\left(\sum_j\alpha_j\mu_j-\mu_0\right)
+\bar Y+\frac{R}{\bar\lambda}.
\label{eq:polar-residual-decomposition}
\end{equation}
The first term has norm $D_K$. For the second, Assumptions (\ref{ass:polar-variance})--(\ref{ass:polar-correlation}) and the nonnegative weights imply
\[
\mathbb E\fnorm{Y_i}^2
\le\tau_W^2\sum_j\alpha_j^2
+\rho\tau_W^2\sum_{j\ne\ell}\alpha_j\alpha_\ell
=\tau_W^2\bigl(S_K+\rho(1-S_K)\bigr)=s_{K,W}^2.
\]
The $Y_i$ are independent with mean zero, so $\mathbb E\fnorm{\bar Y}\le s_{K,W}/\sqrt m$. For the last term, write $\sigma_H^2=m^{-1}\sum_i\fnorm{H^{(i)}-\bar H}^2$. Proposition~\ref{prop:residual} and (\ref{ass:polar-clipping}) give
\[
\frac{\fnorm{R}}{\bar\lambda}\le c_\lambda\sigma_H,
\qquad
\mathbb E\sigma_H
\le\sqrt{\mathbb E\sigma_H^2}
=\sqrt{\frac{m-1}{m}\mathbb E\fnorm{Y_i}^2}
\le\sqrt{\frac{m-1}{m}}\,s_{K,W}.
\]
Finally, $\polar(N)=\polar(N/\bar\lambda)$ because $\bar\lambda>0$. Taking norms and expectations in Eq.~\eqref{eq:polar-residual-decomposition} and applying Lemma~\ref{lem:polar-lipschitz} proves Eq.~\eqref{eq:polar-history-bound}.
\end{proof}

\textbf{When the bound improves.}
Without averaging, $D_0=0$ and $s_{0,W}=\tau_W$. For $\tau_W>0$, averaging gives $s_{K,W}<\tau_W$ when $\rho<1$ and at least two weights are nonzero ($S_K<1$). At the same current state and batch size, use one constant $c_\lambda$ satisfying (\ref{ass:polar-clipping}) for both the averaged and current-only queries. With $a_m=m^{-1/2}+c_\lambda\sqrt{(m-1)/m}$, the right-hand side of Eq.~\eqref{eq:polar-history-bound} is no larger than the current-only bound whenever
\begin{equation}
D_K\le a_m(\tau_W-s_{K,W}).
\label{eq:polar-improvement-condition}
\end{equation}
Strict inequality gives a strictly smaller bound. This compares upper bounds, not actual errors. When $\rho=1$, the variance terms in this bound do not decrease.

If each per-example loss is $\beta$-smooth in the full parameter vector, uniformly over the local randomness, then the block mean is $\beta$-Lipschitz and
\begin{equation}
D_K\le\beta\sum_{j=0}^{K}\alpha_j\|\theta_j-\theta_0\|_2.
\label{eq:polar-staleness}
\end{equation}
This is the smoothness condition used by \citet{xiao2023theory}, with the common conditional evaluation-randomness law specified above. Nearby states control staleness, but do not by themselves imply temporal decorrelation.

\textbf{Proof of Proposition~\ref{prop:residual-envelope}.}
Use the full-query gradients and moment assumptions with $\tau_u$ from Section~\ref{sub:envelope}. Let $u_i=\sum_j\alpha_jg_j(z_i)$, $Y_{i,u}=\sum_j\alpha_j\xi_{j,u}(z_i)$, and $\sigma_u^2=m^{-1}\sum_i\|u_i-\bar u\|_2^2$. The map $n\mapsto\min\{1,C/n\}$ is $1/C$-Lipschitz: its derivative is $0$ for $n<C$ and $-C/n^2\in[-1/C,0)$ for $n>C$. The variance of the $\lambda_i$ is at most their second moment about $\lambda(\|\bar u\|)$, and $|\,\|u_i\|-\|\bar u\|\,|\le\|u_i-\bar u\|$, so $\sigma_\lambda\le\sigma_u/C$. Proposition~\ref{prop:residual} gives $\|R_W\|_F\le\sigma_\lambda\sigma_{H,W}$, and $\sigma_{H,W}\le\sigma_u$ because a block is a sub-collection of coordinates. Hence $\|R_W\|_F\le\sigma_u^2/C$. The full-query moment bounds give $\mathbb E\|Y_{i,u}\|_2^2\le s_{K,u}^2$. Since the means $\sum_j\alpha_j\mu_{j,u}$ cancel in $u_i-\bar u$, the same sample-variance identity used in Theorem~\ref{thm:polar} yields $\mathbb E\,\sigma_u^2=\frac{m-1}{m}\mathbb E\|Y_{1,u}\|_2^2\le\frac{m-1}{m}s_{K,u}^2$. \qed

If the full-query moment bounds in Section~\ref{sub:envelope} hold with equality and $\tau_u>0$, then $\mathbb E\,\sigma_u^2=\frac{m-1}{m}s_{K,u}^2$ exactly. The ratio of the expected empirical variances for averaged and current-only queries is then $\rho+(1-\rho)S_K$.

\end{document}